\documentclass[10pt, abstract, DIV=12]{scrartcl}
\usepackage{authblk}
\usepackage[numbers, sort]{natbib}
\usepackage{mathpazo}
\setkomafont{sectioning}{\bfseries}
\setkomafont{author}{\small}

\usepackage{amsthm}
\newtheorem{theorem}{Theorem}

\newtheorem{lemma}[theorem]{Lemma}
\newtheorem{corollary}[theorem]{Corollary}
\newtheorem{definition}[theorem]{Definition}

\newtheorem{remark}[theorem]{Remark}
\newtheorem{example}[theorem]{Example}

\usepackage{amsmath,amsfonts,bm}

\def\eqref#1{equation~\ref{#1}}
\def\Eqref#1{Equation~\ref{#1}}

\def\1{\bm{1}}

\DeclareMathAlphabet{\mathsfit}{\encodingdefault}{\sfdefault}{m}{sl}
\SetMathAlphabet{\mathsfit}{bold}{\encodingdefault}{\sfdefault}{bx}{n}

\def\sN{{\mathbb{N}}}

\def\sR{{\mathbb{R}}}

\usepackage{hyperref}
\usepackage{url}
\usepackage{graphicx}   
\usepackage{booktabs}   
\usepackage[utf8]{inputenc} 
\usepackage[T1]{fontenc}    
\usepackage{hyperref}       
\usepackage{url}            
\usepackage{amsfonts}       
\usepackage{microtype}      
\usepackage{xcolor}         
\usepackage{placeins}

\usepackage[english]{babel}

\usepackage{booktabs}   
\usepackage{amsmath}    
\usepackage{amssymb}    
\usepackage{amsthm}     
\usepackage{booktabs}   
\usepackage{dsfont}     
\usepackage{etoolbox}   
\usepackage{graphicx}   
\usepackage{hyperref}   
\usepackage{microtype}  
\usepackage{mleftright} 
\usepackage{multirow}   
\usepackage{nicefrac}   
\usepackage{paralist}   
\usepackage{siunitx}    
\usepackage{subcaption} 
\usepackage{tikz}       
\usepackage{pgfplots}   
\usepackage{soul}       
\usepackage{tabu}       
\usepackage{wrapfig}    
\usepackage{xspace}     

\usepackage[linesnumbered,ruled,vlined]{algorithm2e}
\SetKwInput{KwInput}{Input}                
\SetKwInput{KwOutput}{Output}              

\numberwithin{theorem}{section}

\title{The magnitude vector of images}

\date{}

\author[1, 2]{Michael F.\ Adamer}
\author[3]{Edward De Brouwer}
\author[1, 2]{Leslie O'Bray}
\author[1, 2, 4]{Bastian~Rieck}

\affil[1]{Department of Biosystems Science and Engineering, ETH Zurich, 4058 Basel, Switzerland}
\affil[2]{SIB Swiss Institute of Bioinformatics, Switzerland}
\affil[3]{ESAT-STADIUS, KU Leuven, 3001 Leuven, Belgium}
\affil[4]{B.R.\ is now with the Institute of AI for Health, Helmholtz Zentrum München, Neuherberg, Germany}

\begin{document}

\maketitle

\begin{abstract}
The magnitude of a finite metric space has recently emerged as a novel invariant quantity, allowing to measure the \emph{effective size} of a metric space. Despite encouraging first results demonstrating the descriptive abilities of the magnitude, such as being able to detect the boundary of a metric space, the potential use cases of magnitude remain under-explored. In this work, we investigate the properties of the magnitude on images, an important data modality in many machine learning applications. By endowing each individual images with its own metric space, we are able to define the concept of magnitude on images and analyse the individual contribution of each pixel with the magnitude vector. In particular, we theoretically show that the previously known properties of boundary detection translate to edge detection abilities in images. Furthermore, we demonstrate practical use cases of magnitude for machine learning applications and propose a novel magnitude \emph{model} that consists of a computationally efficient magnitude computation and a learnable metric. By doing so, we address the computational hurdle that used to make magnitude impractical for many applications and open the way for the adoption of magnitude in machine learning research.
\end{abstract}

\section{Introduction}

The topology community has recently invested much effort in
studying a newly introduced quantity called \emph{magnitude}
\citep{leinster2010magnitude}. While it originates from category theory,
where it can be seen as a generalisation of the Euler characteristic to
metric spaces, the magnitude of a metric space is most intuitively
understood as an attempt to measure the \emph{effective size} of
a metric space \citep{sorensen1948method}. As a descriptive scalar, this quantity extends the set
of other well known descriptors such as the rank, diameter or dimension.
However, unlike those descriptors, the properties and potential use cases of magnitude are still under-explored.
Because the metric space structure of datasets is a natural object of study when it comes to the understanding of fundamental machine learning concepts such as regularization, magnitude appears like a promising and powerful concept in machine learning: next to its abilities to describe the metric space of whole datasets, the magnitude can also be studied at the sample level, by considering each sample as its own metric space.
Following this line of thought, magnitude vectors were introduced as a way to characterise the contribution of each data sample to the overall magnitude of the dataset, such that the sum of the elements of the magnitude vector amounts to the magnitude. This allowed to assess the individual contribution of each data point and their relative connectivity in the whole dataset. Indeed, magnitude vectors have been shown to detect boundaries of metric spaces, with boundary points exhibiting larger contributions to the magnitude~\citep{bunch2021weighting}.

Building upon these recent advances, we study the concept of magnitude of images, an important data modality for a plethora of machine learning applications. We endow each image with its own metric space and explore the properties of the magnitude and the magnitude vector for different choices of metric space structure. We extend the concept of a boundary of a metric space to images and show that it corresponds to edge detection abilities. We thus investigate the potential of magnitude for edge detection architectures and propose a new \emph{magnitude model}. This model consists of a learnable metric on images followed by an efficient approximation of the magnitude on the learnt metric space. Our experiments show that this architecture is on par with existing edge detection approaches and thus represent a first promising use-case for magnitude in machine learning applications. What is more, we compare the magnitude model and the Sobel filter edge detectors from a topological perspective and find that both filters radically differ, with the magnitude model displaying more cycles and connected components. This points towards a potential complementarity of both approaches. 

Our \textbf{contributions} can be summarised as follows:
\begin{itemize}
    \item We formalise the notion of magnitude vectors for images and investigate the impact of the choice of different metrics on the images. We further derive analytic forms for the magnitude vector of special cases of images.
    \item Based on this formalism, we provide a theoretical framework to understand the edge detection capabilities of magnitude vectors and link it to the previously known interpretation of boundary of a metric space. This framework provides a first basis for theoretically motivating the usage of magnitude in machine learning applications.
    \item We propose new efficient approximations for the computation of magnitude vectors on images, therefore facilitating the usage of magnitude in applications.
    \item We introduce a novel magnitude \emph{model} that consists of the combination of learnable metric on images and an computationally efficient approximation of the magnitude vector of the image with the resulting metric space. We evaluate  the ability of this model to perform edge detection images and show that it compares favourably with existing edge detection implementations. We also evaluate the topological properties of the detected edges and find them substantially different to comparable methods, suggesting a complementarity of the magnitude model with previous works.
\end{itemize}

This paper is organised as follows. In Section \ref{sec:theory} we provide a theoretical framework for the behaviour of the magnitude measure on images. In Section \ref{sec:applications}, we consider the practicalities of computing the magnitude vector of images and provide a speedup algorithm. In Section \ref{sec:experiments} we evaluate the approximation methods described in this paper and perform experiments on the edge detection capabilities of the magnitude vector. Our results are summarised in Section \ref{sec:discussion}.


\section{Theorectical Results}\label{sec:theory}

We start by introducing the essential notions of the theory of magnitude, magnitude measures, and magnitude vectors.
We proceed by laying out how an image can be viewed as a compact metric space and derive explicit formulae for the magnitude measure on the space of one-dimensional images. We further show how the magnitude measure for two-dimensional images can be approximated by one-dimensional images.

\subsection{Mathematical Background}\label{subsec:math_back}

We start by formally introducing the notion of a finite and compact metric spaces.
\begin{definition}
A metric space is an ordered pair $(B,d)$, where $B$ is a finite or compact set and $d$ is a metric on $B$. If $B$ is finite, then we denote the cardinality of $B$ by $|B|$, if $B$ is a compact set, then $|B|$ denotes its dimensionality.
\end{definition}
In many applications, the set $B$ is a set of vectors $B\subset \sR^n$ and the metric considered is the $\ell_p$ norm. In order to define the magnitude of such a space we first define the similarity matrix of a metric space.
\begin{definition}
Given a finite metric space $(B,d)$, its similarity matrix is $\zeta_B$ with entries $\zeta_B(i,j) = e^{-d(B_i,B_j)}$ for $B_i,B_j\in B$.
\end{definition}
We are now in a position to define the magnitude vector and the magnitude of a finite vector space.
\begin{definition}
Consider a finite metric space $(B,d)$ of cardinality $|B|=n$ and a metric $d$. Denote its similarity matrix by
$\zeta_B$ with inverse $\zeta_B^{-1}$. The magnitude vector of element
$B_i$ is given by $w_i=\sum_{j=1}^n\zeta_B^{-1}(i,j)$. Moreover, the magnitude of $(B,d)$, $\textit{mag}_B$ is $\sum_{i,j=1}^n\zeta_B^{-1}(i,j) = \sum_{i=1}^n w_i$.
\label{def:MagVec_finite}
\end{definition}
Not every finite metric space has a magnitude. In particular, the
magnitude is not defined when the similarity matrix is not invertible;
the magnitude therefore characterises the structure of a metric space to
some extent.
For compact metric spaces an analogous notion exists.
\begin{definition}[\citep{willerton2014magnitude}]
Consider a metric space $X = (B,d)$ with a compact set $B$ and a metric $d$. A finite, signed Borel measure $\nu$ is called a \emph{magnitude measure} on $X$ if 
the following relation holds for all $\bm{y}\in X$:
$$
\int_{\bm{x}\in X} e^{-d(\bm{x},\bm{y})}d\nu(\bm{x}) = 1.
$$
Furthermore, the magnitude of $X$ is given by
$$
\text{mag}_X = \int_{\bm{x}\in X} d\nu(\bm{x}).
$$
\label{def:MagVec_infinite}
\end{definition}

We see that the compact metric space and the finite metric space cases are analogous where sums are replaced by integrals and weight vectors are replaced by magnitude measures.

\subsection{An image as a compact metric space}

What is an image? Different fields have produced distinct definitions. In computer science, and computer vision, an image is typically conceptualised as an array of pixels, where each pixel consists of a number of channels (usually 1 or 3), therefore corresponding to a tensor.
In this work, we propose an alternative conception and consider a digital image as a set of points in some ambient space, as formalised below.

\begin{definition}[Digital Image]
A digital image is a set of points $\in\sR^{2+n}$ of the form $\{(i,j,c^{(i,j)}_1,\dots,c^{(i,j)}_n)^T\}$, with $i,j\in\sN$ such that $0<i<w$, $0<j<h$ for some integers $w,h$.
\label{def:digital_image}
\end{definition}

In the days of analogue photography, however, this definition was not so clear cut. There, an image was just a \emph{projection} of three-dimensional space onto a two-dimensional object.
We start with this more analogue definition of an image and define it as a surface in some ambient space. For grey-scale images this would be a two-dimensional surface in three-dimensional space. Note that this surface is not infinite as it is bounded by the edges of the image. With this definition, we can view an image with $n$ colour channels as a compact metric space $X$ of dimension $2$ in some $n+2$ dimensional ambient space. Each point on this image is given by $\bm{x}=(x,y,c_1,\dots,c_n)^T = (x,y,\bm{c})^T\in X$. We also define $n$ maps from the points $(x,y)^T$ to the channels $c_i$, given by $c_i = \phi_i(x,y)$, and denote the vector of $n$ maps as $\Phi$ and the vector of channels as $\bm{c}$. We now define the domain of an an image.

\begin{definition}[Domain]
The \emph{domain} $D$ of an image is the region $0\leq x_1\leq w$, $0\leq x_2\leq h$, where $w$ and $h$ are the width and the height of an image. A point in the domain is denoted by $x\in D$.
\end{definition}

Note that the domain contains all its boundary points. We can now define an analogue image.

\begin{definition}[Analogue Image]
An analogue image is a continuous map on a domain $D$, $\Phi: \sR^2 \to \sR^n$ such that $\Phi(x_1,x_2) = \Phi(x) = \bm{c}$.
\end{definition}

In between the digital and the analogue image, there exists a third way of describing images, which we call \emph{digitised images}. In a digitised image, we still have pixels, however, these have a finite area in $\sR^2$, i.e. each pixel has an area $\Delta x_1 \Delta x_2$ with a constant channel value $\bm{c}$. Therefore, we do not consider continuous surfaces but step functions.

\begin{remark}
It has been pointed out that step functions may not be an appropriate model for images \citep{malik1992finding}. However, for simplicity, we still consider step function images.
\end{remark}

\begin{definition}[Digitised Image]
A digitised image is a step function on a domain $D$, $\Phi_s: \sR^2 \to \sR^n$ such that $\Phi_s(x_1,x_2) = \Phi_s(x) = \bm{c}$.
\end{definition}

Note that digitised images are technically not compact metric spaces, however, as we will show, we can rephrase the problem to alleviate this issue.

\begin{remark}[Notation]
As introduced above, we denote a vector in the \emph{domain} as $x$ with elements $x_i$. In contrast, we use \textbf{bold font} to denote vectors on the image consisting of a point on the domain and all its pixel brightnesses. That is, $\mathbf{x} = (x,\Phi(x))^T$.
\end{remark}

\subsection{The magnitude of images}

The most straightforward approach to define the magnitude of an image is to choose a metric $d$ and calculate the magnitude vector directly on the set of points defined by a digital image as defined in Definition \ref{def:digital_image} using Definition \ref{def:MagVec_finite}. However, this approach is unsatisfactory for both computational and theoretical reasons:
\begin{enumerate}[(i)]
    \item Computationally, this requires the inversion of a $(\# pixels \times \# pixels)$-matrix, for which the computational cost grows cubically with the number of pixels. 
    \item Theoretically, this makes the tracing of how the individual pixel weights are formed very challenging. Indeed, the magnitude computation is global and all pixels can potentially  contribute to the magnitude weight of each pixel.
\end{enumerate}

To address the above limitations, we propose an alternative approach based on continuous images~(analogue and digitised). We then show experimentally that this new alternative accurately recovers the digital scenario, however, with some numerical aberrations.

For each pair of points in the image domain $D$, $x,x'\in [0,w]\times[0,h]$, we can express the corresponding points on the image (the $\sR^{2+n}$ ambient space) by $\bm{x} = (x,\Phi(x))^T$ and $\bm{x}' = (x',\Phi(x'))^T$, or $\bm{x} = (x,\Phi_s(x))^T$ and $\bm{x}' = (x',\Phi_s(x'))^T$, respectively for digitised images. 

By choosing a metric $d$ for the set of tuples $\mathbb{X} = \{\mathbf{x}=(x,\Phi(x)): x\in D\}$, we can view each image as a compact metric space $X = (\mathbb{X},d)$. A natural starting point for our derivations is the definition of the magnitude measure $\nu$ that satisfies:
\begin{equation}
\int_{\bm{x}\in \mathbb{X}} e^{-d(\bm{x},\bm{y})}d\nu(\bm{x}) = 1,
\label{eq:starting_point}
\end{equation}
for all $\bm{y}\in X$.
Substituting with the definition of $\mathbf{x}$, we obtain an explicit integral equation for the magnitude measure of an analogue image:
\begin{equation}
\int_{x\in D} e^{-d((x,\Phi(x))^T,(x',\Phi(x'))^T)}d\nu((x,\Phi(x))^T) = 1,
\label{eq:analogue_image_general}
\end{equation}
for all $x'\in D$. Unfortunately, \Eqref{eq:analogue_image_general} is in general analytically intractable. However, for specific cases, we can obtain an explicit computation, as we venture to show in this work.

\begin{remark}
Note that by using the map $\Phi$, we can rephrase the problem of integration over a (possibly non-compact) space in $\sR^{2+n}$ into an integration over a bounded plane in $\sR^2$ (which is always compact) and a modified metric.
\end{remark}

\subsubsection{1D images}



One-dimensional images are images with a one-dimensional domain $D=[0,w]$. The simplest such image is a constant line segment.
We start by restating the magnitude measure of a line segment as proven in \citep{willerton2014magnitude}.

\begin{theorem}[{\citep[Theorem 2]{willerton2014magnitude}}]
Let $\mu$ be the Lebesgue measure of a line segment $L_{[a,b]}$, $[a,b]$, and let $\delta_a$ and $\delta_b$ be the Dirac measures at the respective end points. Then the magnitude measure $\nu$ on $L_{[a,b]}$ is given by $\nu = \tfrac{1}{2}(\mu + \delta_a + \delta_b)$. Hence the magnitude is simply $$\text{mag}_{L_{[a,b]}} = 1 + \frac{a-b}{2}.$$
\label{thm:mag_line}
\end{theorem}
\begin{proof}
The proof follows from direct integration.
\end{proof}

Equipped with the results on a line segment, we now focus on the meaning of the\newline metric $d((x,\Phi(x))^T,(x',\Phi(x'))^T).$ 
Recall that the domain we are interested in for 1D images are always line segments (which corresponds to a single row of pixels). The brightness and colour channels of the images modulate the typical metrics defined on $D$.
We thus turn to the question of exactly \emph{how} this metric is modified.

In the case of an analogue single channel image, the image surface is diffeomorphic to a bounded plane or, in other words, it is a warped plane. Therefore, two points on the plane are connected via the geodesic distance, or curve length in the case of a one-dimensional image. The influence of choosing a metric, e.g. an $\ell_1$ metric, on the domain is choosing how exactly this distance is calculated. In the case of a one-dimensional domain this is
\begin{subequations}
\begin{align}
    d((x_1,\Phi(x_1))^T,(x_1',\Phi(x_1'))^T) &= \left|\int_{x_1}^{x_1'}\sqrt{1 + \left(\tfrac{d\Phi(y)}{dy}\right)^2}dy\right| \qquad\text{for $\ell_2$,}\\
    d((x_1,\Phi(x_1))^T,(x_1',\Phi(x_1'))^T) &= |x_1-x_1'| + \left|\int_{x_1}^{x_1'}\left|\tfrac{d\Phi(y)}{dy}\right| dy\right| \qquad\text{for $\ell_1$.}\label{eq:modified_metrics_l_1}
\end{align}
\label{eq:modified_metrics}
\end{subequations}
In the case of digitised images, we encounter discontinuous step functions. In this case, the curve length is the usual distance on a plane plus the absolute height of the steps between two points as can be seen from direct integration of \eqref{eq:modified_metrics}.

\begin{example}
Consider the $\ell_1$ metric, on an analogue image on a domain $D = [0,w]$. 
We can derive the magnitude measure of
\begin{enumerate}[(i)]
    \item the constant image $\Phi(x_1) = \text{const.}$,
    \item the single channel ($n=1$) line image $\Phi(x_1) = \alpha x_1 + \text{const}$, where $\alpha\in\sR$.
\end{enumerate}
Note that the constant image is a special case of the line image with $\alpha=0$. Substituting the line image into equation \eqref{eq:modified_metrics_l_1}, we obtain
$$
 d((x_1,\Phi(x_1))^T,(x_1',\Phi(x_1'))^T) = |x_1-x_1'| + |\alpha||x_1-x_1'|.
$$
Therefore, equation \eqref{eq:analogue_image_general} becomes
\begin{equation*}
\int_{0}^w e^{-(|\alpha|+1)|x_1-x_1'|}d\nu(\bm{x}) = 1,
\end{equation*}
for all $x'\in [0,w]$.
Using Theorem \ref{thm:mag_line}, we conclude that the magnitude measure for constant or line images is given by $\nu(X) = \tfrac{1}{2}(|\alpha|+1)(\mu(x_1)+\delta_0+\delta_w)$.
\label{ex:line_image}
\end{example}

For more complex analogue images, or metrics other than $\ell_1$, finding a magnitude measure becomes analytically intractable. Therefore, we turn our focus to digitised images. We begin by considering a one-dimensional image, $D = [0,w]$, with a single channel and two pixels. Note, that this scenario corresponds to the usual step function where the step is located at $w/2$.

\begin{lemma}
Let $\phi_s(x_1) = \gamma H(x_1-w/2) + c$, where $H(\cdot)$ is the Heaviside function with convention $H(0) = 0$, constants $w,c\in\sR_+$, and $\gamma\in\sR$. The magnitude measure of the metric space defined by $\phi(x_1)$ in the domain $D = [0,w]$ with $\ell_1$ metric is given by
$$
\nu(x_1) = \frac{1}{2}\left(\mu + \delta_0 + \delta_w + (1-e^{-|\gamma|})\delta_{w/2}\right).
$$
\label{lem:two_pixel_1d}
\end{lemma}
\begin{proof}
Calculating the curve length from \eqref{eq:modified_metrics_l_1} and substituting the magnitude measure into \eqref{eq:starting_point}, we obtain
$$
I = \int_0^w e^{-|x_1-x_1'|-|\gamma||H(x_1-w/2)-H(x_1'-w/2)|} \underbrace{d\left[\frac{1}{2}\left(\mu + \delta_0 + \delta_w + (1-e^{-|\gamma|})\delta_{w/2}\right)\right]}_{d\nu(x_1)}.
$$
We now distinguish two cases, $x_1' \leq w/2$ and $x_1' > w/2$.

\paragraph{Case 1 ($x_1' \leq w/2$):} The integral simplifies to
$$
\int_0^w e^{-|x_1-x_1'|-|\gamma|H(x_1-w/2)} d\nu(x_1).
$$
We now split the interval $[0,w]$ into three parts, $[0,x_1'),[x_1',w/2],(w/2,w]$ and compute the three integrals $I + I_1 + I_2 + I_3$. From the first interval, we obtain $I_1 = \tfrac{1}{2}$.The second interval gives $I_2 = \tfrac{1}{2}(1-e^{-|\gamma|-(w/2 -x_1')}$, and the third interval integrates to $I_3 = \tfrac{1}{2}e^{-|\gamma|-(w/2 -x_1')}$. We then have $I_1 + I_2 + I_3 = I = 1$.

\paragraph{Case 2 ($x_1' > w/2$):} The integral simplifies to
$$
\int_0^w e^{-|x_1-x_1'|-|\gamma|(1-H(x_1-w/2))} d\nu(x_1).
$$
Again, we divide the domain into three parts, $[0,w/2),[w/2,x_1'],(x_1',w]$ and compute the integrals $I_1,I_2,I_3$. This gives $I_1 = \tfrac{1}{2}e^{-|\gamma|}e^{(w/2-x_1')}$, $I_2 = \tfrac{1}{2}(1-e^{-|\gamma|}e^{(w/2 -x_1')})$, and $I_3 = \tfrac{1}{2}$. We also obtain $I_1+I+2+I_3=I=1$
\end{proof}

\begin{remark}[Reflection property]
Note that the second case is also covered by reflecting the function about the $y$-axis and integrating from $w$ to $0$, i.e. (1) let $z = -x$, (2) let $\int_0^{-w} \to -\int_{-w}^0$ and (3) shift $z \to z+w$. This leaves the integral $I$ invariant and we refer to this property as the \emph{reflection property}.
\label{rem:reflection_property}
\end{remark}

We can extend Lemma \ref{lem:two_pixel_1d} to many pixels in a one-dimensional digitised image via induction.

\begin{theorem}
Let $\phi_s(x_1) = \sum_{i=1}^{m-1}\gamma_i H(x_1-(i/m)w) + c$, where $H(\cdot)$ is the Heaviside function with convention $H(0) = 0$, $w,c\in\sR_+$, and $\gamma_i\in\sR$. The magnitude measure of the metric space induced by $\phi(x_1)$ in the domain $D = [0,w]$ with $\ell_1$ metric is given by
$$
\nu(x_1) = \frac{1}{2}\left(\mu + \delta_0 + \delta_w + \sum_{i=1}^{m-1}(1-e^{-|\gamma_i|})\delta_{(i/m)w}\right).
$$
\label{thm:1d_images}
\end{theorem}
\begin{proof}
The proof follows the proof of Lemma \ref{lem:two_pixel_1d}. We first consider three pixels with the curve length metric of \eqref{eq:modified_metrics},
\begin{align*}
I=&\int_0^w e^{-|x_1-x_1'|-|\gamma_1(H(x_1-w/3)-H(x_1'-w/3)| + |\gamma_2(H(x_1-2w/3)-H(x_1'-2w/3))|} \times\\
&\underbrace{d\left[\frac{1}{2}\left(\mu + \delta_0 + \delta_w + (1-e^{-|\gamma_1|})\delta_{w/3}\right)+(1-e^{-|\gamma_2|})\delta_{2w/3}\right]}_{d\nu(x_1)}.    
\end{align*}
Now, we consider three cases $x_1'\leq w/3$, $w/3 < x_1'\leq 2w/3$, and $2w/3 < x_1'\leq w$.

\paragraph{Case 1 ($x_1' \leq w/3$):} We split the domain into four parts, $[0,x_1'),[x_1',w/3],(w/3,2w/3]$, and $(2w/3,w]$ and compute the four resulting integrals $I_1, I_2, I_3$, and $I_4$. By direct integration, the integrals evaluate to $I_1 = \tfrac{1}{2}$, $I_2 = \tfrac{1}{2}(1-e^{-|\gamma_1|-(w/3 -x_1')})$, $I_3 = \tfrac{1}{2}(e^{-|\gamma_1|-(w/3-x_1')} - e^{-|\gamma_1|-|\gamma_2|-(2w/3 - x_1')})$, and $I_4 = \tfrac{1}{2}e^{-|\gamma_1|-|\gamma_2|-(2w/3 - x_1')}$. Summing the integrals, we obtain $I_1+I_2+I_3+I_4 = I = 1$.

\paragraph{Case 2 ($w/3 < x_1'\leq 2w/3$):} We divide the domain into the parts, $[0,w/3),[w/3,x_1'],(x_1',2w/3)$, and $[2w/3,w]$ and compute the four resulting integrals $I_1, I_2, I_3$, and $I_4$. These evaluate to $I_1 = \tfrac{1}{2}e^{-|\gamma_1|-(x_1'-w/3)}$, $I_2 = \tfrac{1}{2}(1-e^{-|\gamma_1|-(x_1'-w/3)})$, $I_3 = \tfrac{1}{2}(1-e^{-(2w/3-x_1')})$, and $I_4 = \tfrac{1}{2}e^{-(2w/3 - x_1')}$. Summing the integrals, we obtain $I_1+I_2+I_3+I_4 = I = 1$.

\paragraph{Case 3 ($x_1' > 2w/3$):} Similarly, we consider the integrals on $[0,w/3),[w/3,2w/3),[2w/3,x')$, and $[x',w]$. The four integrals evaluate to 
\begin{align*}
    I_1 &= \tfrac{1}{2}e^{-|\gamma_1|-|\gamma_2| - (x_1'-w/3)},\\
    I_2 &= \tfrac{1}{2}(e^{-(x_1'-2w/3)-|\gamma_2|}-e^{-|\gamma_1|-|\gamma_2|-(x_1'-w/3)}),\\
    I_3 &= \tfrac{1}{2}(1-e^{-(x'-2w/3)-|\gamma_2|}),\\
    \intertext{and}
    I_4 &= \tfrac{1}{2}.
\end{align*}
Again, we obtain $I_1+I_2+I_3+I_4 = I = 1$.

The generalisation to $m$ step functions follows by induction.
\end{proof}

\begin{remark}[Boundary effects]
Note that Theorem \ref{thm:1d_images} implies that for any image with $\ell_1$ metric there are always ``boundary effects'' meaning that irrespective of the actual value of the image, the largest values always occur at the domain boundary. In practical computer vision applications this is not an issue, as we can use the common technique of padding the image for magnitude calculations and cropping the boundary afterwards.  
\end{remark}

\begin{corollary}[Multi-channel 1D-images]
Theorem \ref{thm:1d_images} can be applied to multi-channel one-dimensional images,
\begin{align*}
    \Phi(x_1)
    =
    \begin{pmatrix}
    \phi_{s,1}(x_1)\\
    \vdots\\
    \phi_{s,n}(x_1)
    \end{pmatrix}
    = 
    \begin{pmatrix}
    \sum_{i=1}^{m-1}\gamma_{i,1} H(x_1-(i/m)w) + c_1\\
    \vdots\\
    \sum_{i=1}^{m-1}\gamma_{i,n} H(x_1-(i/m)w) + c_n\\
    \end{pmatrix},
\end{align*}
with a magnitude measure
$$
\nu(x_1) = \frac{1}{2}\left(\mu + \delta_0 + \delta_w + \sum_{i=1}^{m-1}\left(1-e^{-\sum_{j=1}^n|\gamma_{i,j}|}\right)\delta_{(i/m)w}\right).
$$
\label{cor:colour_images}
\end{corollary}

An illustration of the results of this section are provided in Figure \ref{fig:theoretical}. To obtain the numerical magnitude we treated the step function as a one-dimensional digital image and proceeded via matrix inversion. The theoretical magnitude is calculated from Corollary \ref{cor:colour_images}. It can be seen that our results are generally in good agreement with numerical calculations. There are, however, some minor differences in the numerical results and our theoretically obtained magnitude measure. We attribute this to two factors, namely numerical inaccuracies in the matrix inversion and discretisation effects. Discretisation effects occur due to the fact that we consider a finite set of points and, therefore, any infinite step necessarily needs to be approximated via a steep, but finite step.

\begin{figure}
    \centering
    \includegraphics[width=.7\textwidth]{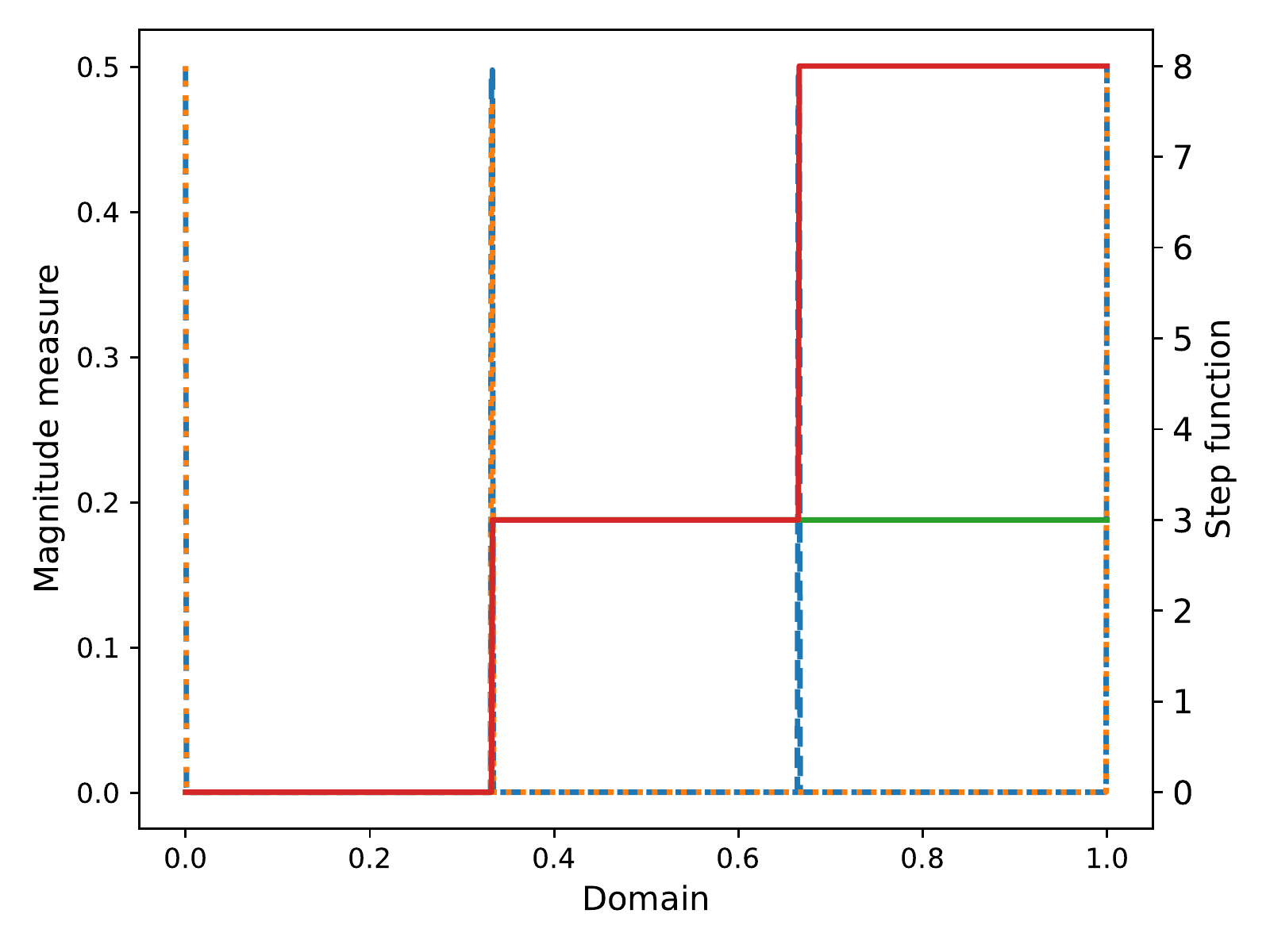}
    \caption{An illustration of the magnitude calculation of a two-channel, two-pixel, one-dimensional image. The solid lines represent the step functions, the dashed blue line is the numerical magnitude and the dotted orange line represents the theoretical magnitude.}
    \label{fig:theoretical}
\end{figure}

\begin{remark}
Note that one-dimensional images can also be viewed as time series, although, the magnitude measure ignores the arrow of time.
\end{remark}

\begin{remark}
Instead of step functions, we could have chosen a piece-wise linear interpolation between the pixels to create a continuous surface. The proofs for this construction (e.g. a piece-wise linear continuous curve in one-dimension) are analogous to the step function case. In fact, the $\ell_2$ curve-length is also tractable. However, piece-wise linear functions are not good models for images.
\end{remark}

\subsubsection{2D images}\label{subsec:2d_images}

Equipped with exact calculations of the magnitude measure of digitised images in one dimension, we now aim to generalise these results to two-dimensional images. 
We first extend Theorem \ref{thm:mag_line} to a bounded plane with $\ell_0$ (Hamming) and $\ell_1$ (Manhattan) metrics analogously.
\begin{theorem}
Let $\mu$ be the Lebesgue measure on the real line and $\nu(L_{[a,b]})$ be the magnitude measure on a line segment $L_{[a,b]}$. Then a magnitude measure on a bounded Plane $P_{[a,b]\times[c,d]}$ is given by
\begin{enumerate}[(i)]
    \item for $\ell_0$: $\nu(\bm{x}) = \tfrac{\mu}{(b-a)(d-c)}$,
    \item for $\ell_1$: $\nu(\bm{x}) = \nu(L_{[a,b]})\nu(L_{[c,d]})$.
\end{enumerate}
\label{thm:mag_plane}
\end{theorem}
\begin{proof}
Note that in both cases we can find a magnitude measure $\nu(\bm{x})$ such that the integral can be expressed as
\begin{equation}
    \int_{x\in [a,b]} e^{-d(x_1,x_1')}d\nu(x_1)\int_{x_2\in[c,d]}e^{-d(x_2,x_2')}d\nu(x_2) = 1, \tag{$\star$}
    \label{eq:proof1}
\end{equation}
for $x\in P_{[a,b]\times[c,d]}$.
For the Hamming metric we note that it is only $\neq 0$ on a set of Lebesgue measure $0$, namely $x_1=x_1'$ or $x_2=x_2'$ respectively. Therefore, the integrand equals $1$ and we can express the magnitude measure as $\nu(x_1) = \mu(x_1)/(a-b)$ and $\nu(x_2) = \mu(x_2)/(c-d)$. Substituting into \eqref{eq:proof1}, we obtain
$$
\frac{1}{(b-a)(d-c)}\int_{x\in [a,b]} d\mu(x_1)\int_{y\in[c,d]}d\mu(x_2) = 1,
$$
which completes the proof.
In the case of the Manhattan metric, we note that each integral equals the integral over a real line with magnitude measure $\nu(L_{[a,b]})$ and $\nu(L_{[c,d]})$ respectively.
\end{proof}

A more general result of Theorem \ref{thm:mag_plane} has also been obtained in \citep[Proposition 3.7]{leinster2017magnitude}. We illustrate the ramifications of Theorem \ref{thm:mag_plane} by generalising \ref{ex:line_image} to two-dimensional domains.

\begin{remark}
Interestingly, it seems that the Hamming distance does not suffer from boundary effects, however, due to the nature of this distance one also reduces the information carried by the metric. Furthermore, the Hamming distance is not robust to noise, i.e. small perturbations in the pixel values may have large effects on the pixel distance.
\end{remark}

\begin{example}
\begin{figure}
    \centering
    \subfloat[A constant image with $\Phi(x) = 123$.]{\includegraphics[width=.3\textwidth]{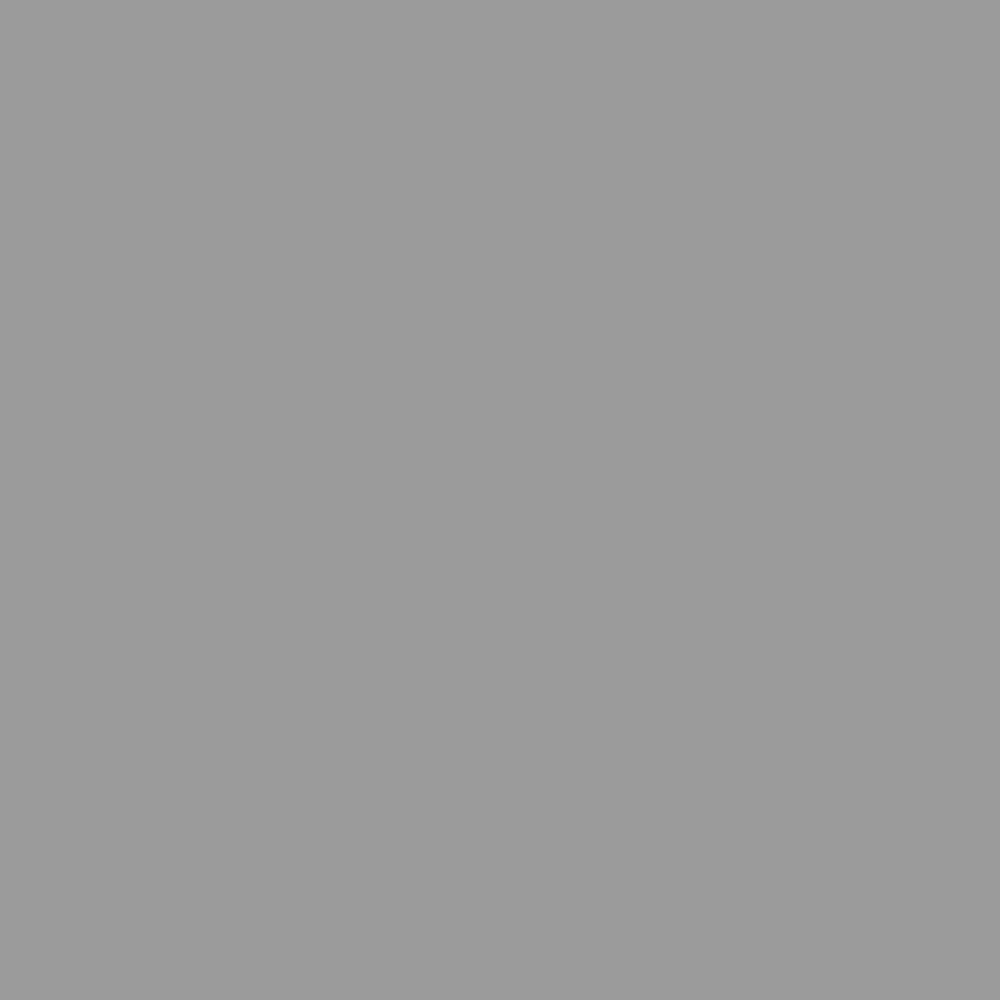}\label{subfig:const}}\hspace{5cm}%
    \subfloat[A line image with $\Phi(x) = x_1$.]{\includegraphics[width=.3\textwidth]{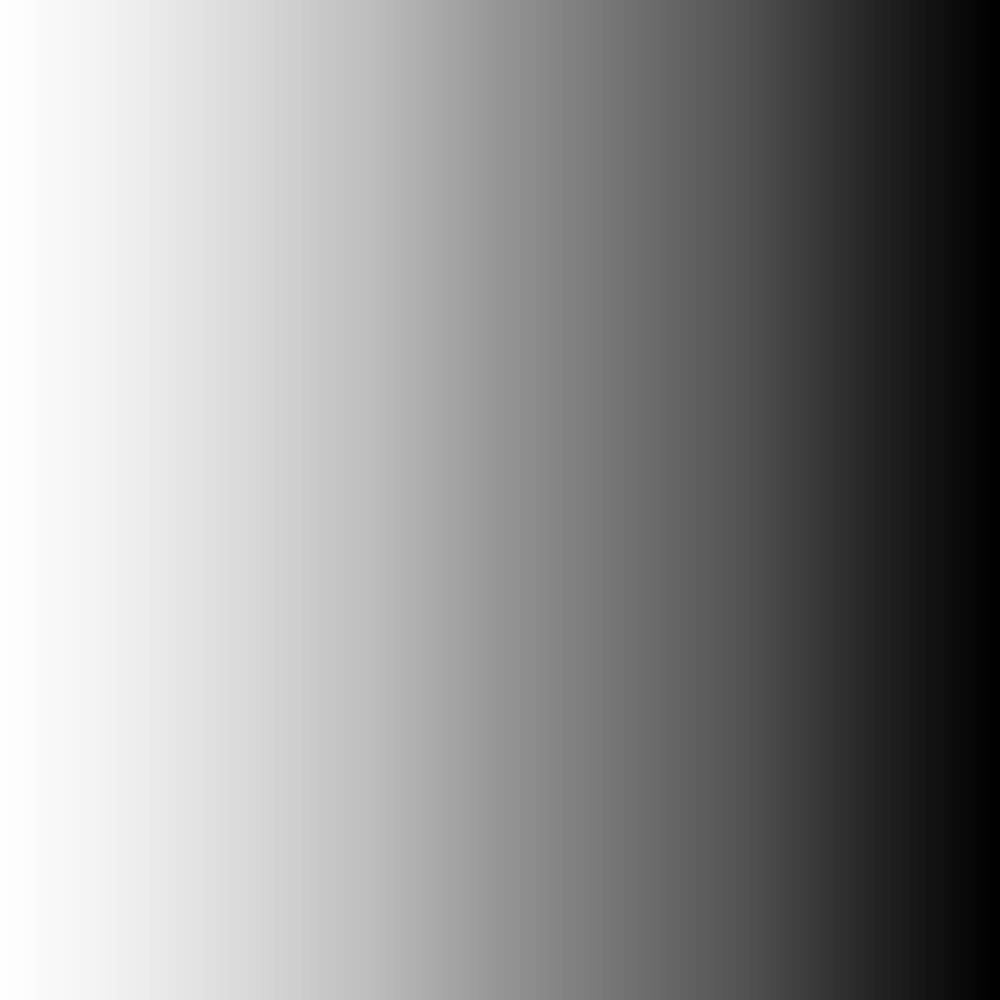}\label{subfig:line}}
    \caption{Two 1D images. The brightness channel is constant across at least one of the dimensions.}
    \label{fig:simple_images}
\end{figure}

Consider the $\ell_1$ metric, on an analogue image on a domain $D = [0,w]\times[0,h]$. 
We can derive the magnitude measure of
\begin{enumerate}[(i)]
    \item the constant image in Subfigure \ref{subfig:const} with $\Phi(x) = \text{const.}$,
    \item the single channel ($n=1$) line image in Subfigure \ref{subfig:line} with $\Phi(x) = \alpha x_1 + \text{const}$, where $\alpha\in\sR$.
\end{enumerate}
Note that the constant image is a special case of the line image with $\alpha=0$. Substituting the line image into equation \eqref{eq:modified_metrics_l_1}, we obtain
$$
 d((x,\Phi(x))^T,(x',\Phi(x'))^T) = |x-x'| + |\alpha||x_1-x_1'|.
$$
Therefore, equation \eqref{eq:analogue_image_general} becomes
\begin{equation*}
\int_{x\in D} e^{-(|\alpha|+1)|x_1-x_1'|-|x_2-x_2'|}d\nu(\bm{x}) = 1,
\end{equation*}
for all $x'\in D$.
Using Theorem \ref{thm:mag_plane}, we conclude that the magnitude measure for constant or line images is given by $\nu(X) = (|\alpha|+1)\nu(D)$.
Note that these results generalise straightforwardly to multi-channel images by letting, for example, $\Phi(x) = (\phi_1(x),\dots,\phi_n(x))^T = (\alpha_1x_1,\dots,\alpha_n x_n)^T$.
\end{example}

The obvious difficulty to obtain results for more general images is that the geodesic distance is not as straightforward to calculate. In fact, for general images, this is analytically intractable. A simplifying assumption we can make is to consider a rank$-1$ approximation of an image. In this case, the image is outer product of two one-dimensional (digitised) images. Therefore, we can use Theorem \ref{thm:mag_plane} and Theorem \ref{thm:1d_images} to derive the magnitude measure.

\begin{corollary}
Consider two one-dimensional images, $I_1 = (x_1,\phi_s(x_1))^T$, on domain $D_1 = [0,w]$ and $\phi_s(x_1) = \sum_{i=1}^{m-1}\gamma^{(1)}_i H(x_1-(i/m)w) + c^{(1)}$ and $I_2 = (x_2,\phi_s(x_2))^T$, on domain $D_2 = [0,h]$ and $\phi_s(x_2) = \sum_{i=1}^{m-1}\gamma^{(2)}_i H(x_2-(i/m)h) + c^{(2)}$. Let $I$ be the rank$-1$ two-dimensional image formed by $I = I_1 \otimes I_2$. Then $$\nu(x) = \nu^{(1)}(x_1)\nu^{(2)}(x_2),$$ where $\nu^{(i)}(x_i)$ is the one-dimensional magnitude measure of the image $I_i$.
\label{cor:rank_1_2d_images}
\end{corollary}
\begin{proof}
This follows from Theorem \ref{thm:mag_plane} (ii) and Theorem \ref{thm:1d_images}.
\end{proof}

Again, the above corollary can be straightforwardly generalised to colour images using Corollary \ref{cor:colour_images}.
We note that the rank-1 approximation does not hold for almost any image and, therefore we consider another approximation which we call the \emph{independence approximation}.

In the independence approximation we treat each pixel in a digitised image as if it were a pixel in a rank-$1$ image. That is, for a given location $(x_1'',x_2'')\in\sR^2$, we apply corollary \ref{cor:rank_1_2d_images} to obtain a local magnitude measure of the pixel $$\nu(x)|_{(x_1'',x_2'')} = \nu^{(1)}(x_1)|_{x_2''}\:\nu^{(2)}(x_2)|_{x_1''}.$$ Therefore, we only consider the step functions at the edges of each pixel to obtain a weight for the pixel. Even though, this is not a global magnitude measure, it is a reasonable approximation in practise (see Subsection \ref{subsec:benchmarks}) and, since it only relies on local pixel information, is computationally very efficient.

\subsection{Interpretation of the magnitude measure}

In the previous sections we calculated explicit magnitude measures for digitised images with $\ell_1$ metric. In this section, we interpret the meaning of these measures in the context of computer vision. First, we observe that the magnitude measure is a local property, that is, it only depends on the immediate neighbourhood of the point $\bm{x}$ in a domain $D$. This is holds exactly for one-dimensional images and we show, based on one-dimensional considerations, that this also holds at least approximately for two-dimensional images. These results can be used in constructing efficient algorithms when applying our results from \emph{digitised} images to \emph{digital} images. In Figure \ref{fig:theoretical} we empirically show that the values of the magnitude measure calculated from our formulae (based on compact metric spaces) are reproduced in the numerical calculations on finite metric spaces modulo some numerical ``discretisation effects''.

Next, we investigate a potential interpretation for the value of the magnitude measure of a digitised image. Note that in the absence of any steps, it is just a constant equal to half the Lebesgue measure (ignoring boundary effects). Analogously, in numerical experiments the constant is determined by the grid spacing. The magnitude measure has a value larger than this constant only when steps are present, in other words, when the pixel brightness changes. In computer vision a ``rapid change in pixel brightness'' is usually referred to as an \emph{edge} in an image and algorithms whose aim is to find edges in images are called ``edge detectors''. Therefore, we argue that computing the magnitude measure (or vector) of an image performs an edge detection task as it is large in the presence of an edge.

To determine how large a step needs to be in order to count as an ``edge'', we recall the definition of an exponential probability distribution. Let $Z\in[0,\infty)$ be a random variable and $\lambda > 0$. The exponential distribution is given by a probability density function (PDF)
\begin{equation}
    p_Z(z;\lambda) = \lambda e^{-\lambda z},
\end{equation}
with cumulative density function (CDF)
\begin{equation}
    P(z \leq Z;\lambda) = 1 - e^{-\lambda z}.
\end{equation}
If we let $\lambda = 1$ and $x = |\gamma_i|$, then we notice that the prefactor to the singular part of the magnitude measure at the step locus has the form of the CDF of an exponential distribution. In the case of two-dimensional images, we conjecture that a multivariate exponential distribution of the from $p(\bm{z}) \sim e^{-f(\bm{z})}$, where $f(\cdot)$ is any continuous function, needs to be considered. Both approximations we introduced in Subsection \ref{subsec:2d_images} can be considered as a probabilistic independence assumption, i.e. $p(\bm{z}) = p(z_1)p(z_2)$. Using this interpretation, we can consider any threshold for edges as the probability of a step being smaller than the threshold and in Subsection \ref{subsec:edge_detection}, we implement a magnitude-based edge detector. This application immediately leads to our main machine learning task of this paper.

\begin{enumerate}
    \item[\textbf{Question:}] Given that the $\ell_1$ magnitude vector has edge detection capabilities, can we \emph{learn} a metric which improves the edge detection of the magnitude vector?
\end{enumerate}

To answer this question, we first need to investigate efficient ways of calculating the magnitude vector of an image.


\section{Speedup Algorithms and Learnable Metrics}\label{sec:applications}

In this section, we introduce two important tools in order to make magnitude computations more accessible and applicable. The first is an algorithm to speed up magnitude vector calculations based on the reasoning of the previous sections. The second is a neural network architecture which serves as a few-shot deep-learning-based edge detector.

\subsection{Efficient approximations of the magnitude vector of images}

Although already mentioned, we briefly reiterate that a major speedup (at the cost of accuracy) follows directly from \ref{cor:rank_1_2d_images}. One can use a Fourier filter (which have efficient implementations in every major programming language) to calculate the step heights in the $x_1$- and $x_2$-directions and transform it to the approximate magnitude measure. 

The second potential speedup is described in Algorithm \ref{alg:speedup} and also uses the locality property of the magnitude measure. In particular, it is a divide-and-conquer algorithm, where the image is first split into several overlapping patches (to account for boundary effects). Then the magnitude vector is calculated via matrix inversion, the appropriate boundaries are removed and the resulting patches are stitched together again. Therefore, the run time is linear in the number of patches and cubic in the number of pixels in a patch, since matrix inversion is $\mathcal{O}(n^3)$ for a matrix of size $n$.
While theoretical intuition for the correctness of this algorithm is presented in the previous sections, we provide further empirical evidence and runtime comparisons in Subsection \ref{subsec:benchmarks}.

\begin{algorithm}[!ht]
\DontPrintSemicolon
  
  \KwInput{A digital image (img) tensor $(c\times h\times w)$, a metric $d$, a patch size $(h_p,w_p)$, an overlap $\delta$}
  \KwOutput{A magnitude vector as an $(h\times w)$ tensor.}
  \tcc{First split the image into $n$ overlapping patches}
  \textbf{zeroPad}(img): $(c\times h\times w)\to (c\times h+\delta\times w+\delta)$\\
  patch\textunderscore array = \textbf{patchImg}(img): $(c\times h+\delta\times w+\delta) \to (n \times c\times h_p+\delta\times w_p+\delta)$

  \tcc{Find magnitude vector of each patch}
  mag\textunderscore patches = []\\
  \For{patch in patch\textunderscore array}
  {
    zeta = \textbf{getSimilarityMatrix}(patch)\\
    zeta\textunderscore inverse = \textbf{invert}(zeta)\\
    mag\textunderscore vec = \textbf{sumCols}(zeta\textunderscore inverse)\\
    mag\textunderscore patch = \textbf{reshapeAndCropBoundary}(mag\textunderscore vec)\\
    \textbf{append} mag\textunderscore patch \textbf{to} mag\textunderscore patches
  }
  
  \tcc{Now stitch the patches}
  out = \textbf{stitchPatches}(mag\textunderscore patches)
  
  \Return{out}
\caption{Heuristic speedup}
\label{alg:speedup}
\end{algorithm}

\subsection{A pullback metric for edge detection}\label{subsec:metric_learning}

We now return to the machine learning question posed at the end of Section \ref{sec:theory}. Can we learn a metric to improve the edge detection of the magnitude vector? This task can be loosely placed in the subfield of machine learning called metric learning \citep{kaya2019deep}. Traditionally, metric learning is a supervised learning technique which tries to find a metric which minimises the distance between to related points (i.e. points that are in the same class) and maximises the distance between points which are unrelated (i.e. are in a different class). Many techniques for metric learning have been developed including triplet losses \citep{ge2018deep} and triplet networks \citep{hoffer2015deep}.

In the case of optimising the magnitude vector one can think of the label as a one-channel (grey scale or binary) image which labels the ground truth we are trying to approximate, e.g. a manually annotated edge map in an edge detection dataset. Although this ground truth image is a per-pixel labelling it is not straightforward to use it as a label for a point in the input metric space (recall, that the similarity matrix has to be inverted and summed). Therefore, ``classical'' metric learning techniques cannot be applied and an alternative route needs to be taken. First, we formulate the learning task. Given a set $B$, we want to find a metric $d$ such that
\begin{equation}
   \zeta_B^{-1}\mathbb{1} = \bm{y},
\end{equation}
where $\bm{y}$ is the ground truth label and $\zeta_B$ is the similarity matrix of Definition \ref{def:MagVec_finite}. Note that finding functions which are guaranteed to be metrics (in particular, ones that fulfil the triangle inequality) is a highly non-trivial task. Before considering this issue, we proceed by defining a loss function.
One possible loss function for this learning task is the $\ell_2$ loss or mean sqaured error (MSE) loss, 
\begin{equation}
L = (\zeta_B^{-1}\mathbb{1} - \bm{y})^2 .
\label{eq:MSE_Loss}
\end{equation}
Notice that calculating the MSE loss involves a matrix inversion whose computational cost is prohibitive for large images. A straightforward application of Algorithm \ref{alg:speedup} is also not possible as we have no theoretical guarantees that the learnt metric has the same locality properties as the $\ell_1$ metric. Therefore, we need to restrict the function classes which we approximate. A possible solution presents itself in the form of a pullback metric.
\begin{definition}[Pullback metric]
Let $X$ and $Y$ be two metric spaces and $f: X\to Y$ be an injective function. Let $d$ be a metric on $Y$, then the pullback metric is a metric on $X$ given by $$(f^* d)(\bm{x}^{(1)},\bm{x}^{(2)}) = d(f(\bm{x}^{(1)}),f(\bm{x}^{(2)})),$$ for $\bm{x}^{(1)},\bm{x}^{(2)}\in X$.
\label{def:pullback_metric}
\end{definition}

Suppose we let $Y$ be the original image metric space and $d$ is the $\ell_1$ metric. Then, we can immediately apply Algorithm \ref{alg:speedup} to make the loss computations tractable. The machine learning task also reduces to finding an injective function $f: Y\to X$ (usually referred to as an embedding) such that $f$ is injective. Observing that a function $f: X\to Y$ is injective if there exists a function $g: Y\to X$ such that for all $\bm{x}\in X$ $g(f(x))=\bm{x}$, we can parameterise the function $f$ by an autoencoder neural network. Fundamentally, there exist two different autoencoder architectures, compressive autoencoders which approximate a function $\sR^n\to\sR^m$ where $m<n$, and expansive autoencoders for which $m\geq n$. Most autoencoders studied in the literature are compressive autoencoders due to their favourable theoretical properties and their ability to perform non-linear dimensionality reduction. One well-known disadvantage of expansive autoencoders is the fact that without further constraints they can learn the identity function, i.e. $f(\bm{x}) = g(\bm{x}) = \bm{x}, \forall \bm{x}\in X$. Furthermore, in most machine learning applications one would like to \emph{reduce} the dimensionality of the data, not \emph{expand} it. In the case of metric learning in the magnitude setting, however, both properties of expansive autoencoders are favourable as:
\begin{enumerate}[(i)]
    \item The input dimensionality of the model is already low: $n+2$, where $n$ is the number of channels (usually $1$ or $3$).
    \item If the $\ell_1$ metric in the input space is the best metric, then we want our model to be able to learn the identity mapping.
    \item the MSE loss of \eqref{eq:MSE_Loss} is a natural regularizer for the latent space of the autoencoder.
\end{enumerate}

Taking these benefits into account, we design a magnitude edge-detector with an expansive autoencoder in the next section.

\section{Experiments}\label{sec:experiments}

In this section we perform experiments to validate our theoretical claims and investigate the power of the magnitude measure as an edge detector.

\subsection{Datasets}

For our experiments we used the BIPED dataset version 2 \citep{poma2020dense}. BIPEDv2 is a new benchmark dataset specifically designed for edge detection. BIPEDv2 consists of $250$ real-world images and annotated ground truths. The dataset is split into a training dataset of $200$ annotated images and a test dataset of $50$ annotated images. The resolution of the images is $1280\times 720$ which results in almost one million pixels per image. The ground truth annotations have been generated by computer vision experts and moderated by an administrator. The viability of the ground truths have also been confirmed using machine learning methods. 

\subsection{Accuracy of different magnitude approximations}\label{subsec:benchmarks}

The first set of experiments we performed were to empirically validate and investigate the various approximate methods for calculating the magnitude vector of images introduced in the paper. All of our calculations are performed on the test set of BIPEDv2. In order to generate a ground truth magnitude, we rescale the original image to a resolution of $200\times 200$. This resulted in needing to invert a $40 000 \times 40 000$-matrix, which is feasible on current machines. We then tested our three approximations, namely
\begin{enumerate}
    \item the rank-1 approximation,
    \item the independence approximation,
    \item and the patched Algorithm \ref{alg:speedup}.
\end{enumerate}

To evaluate the performance, we first min-max scale the ground truth and the approximated images such that the magnitude values are between zero and one. Then we calculate three performance metrics, namely the maximum absolute deviation between the ground truth and the approximated magnitude vector ($\ell_\infty$ distance), the normalised Fr\"obenius norm given by
\begin{equation}
    \mathrm{error} = \frac{\sum_i(\text{ground truth} - \text{magnitude vector}_\text{approx})^2}{\sum_i(\text{ground truth})^2},
\end{equation}
and the correlation between the two images. The results are presented in Figure \ref{fig:benchmark}

\begin{figure}
    \centering
    \subfloat[A benchmark of the approximation quality of the proposed algorithms]{\includegraphics[width=.5\textwidth]{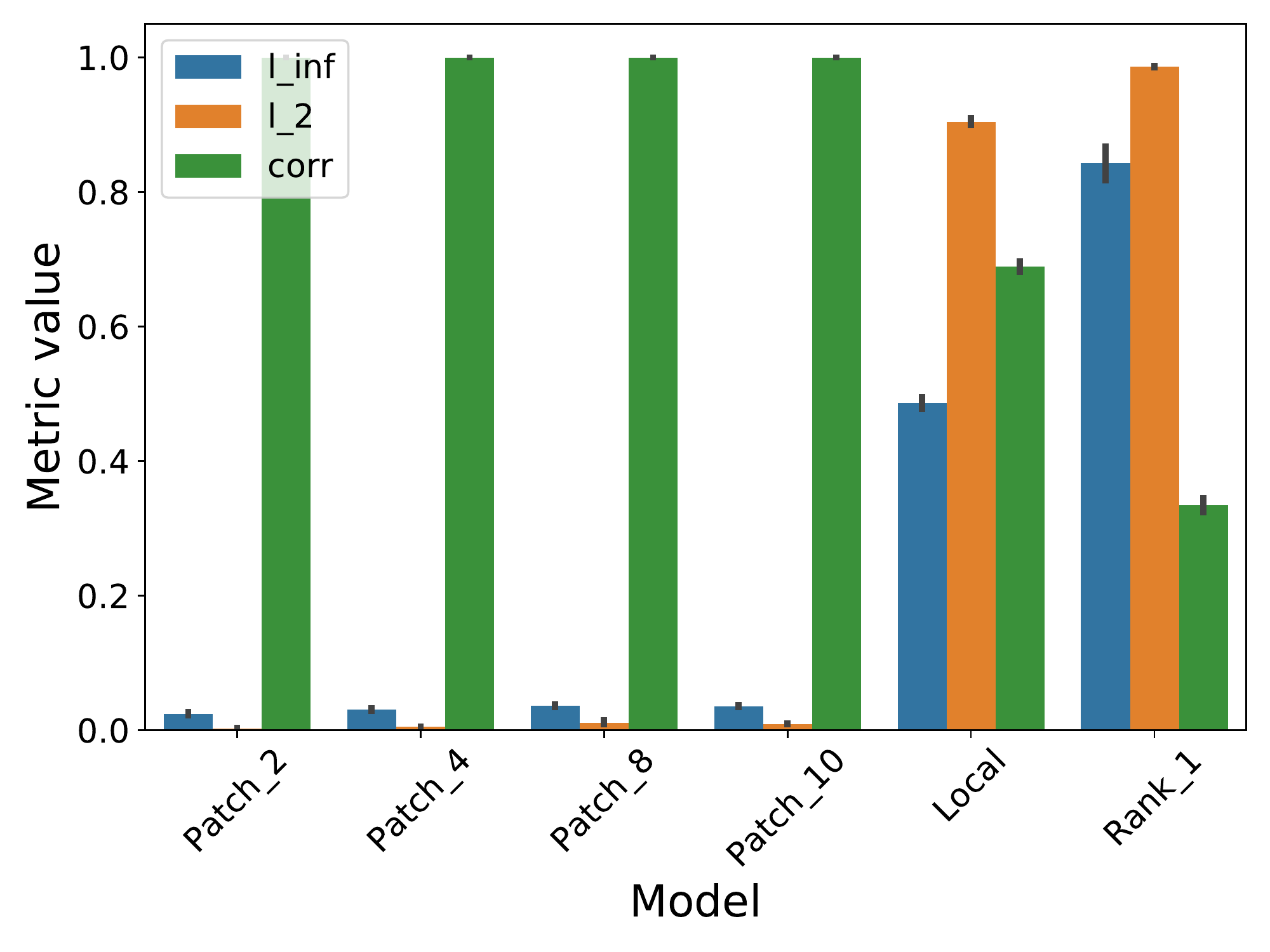}}
    \subfloat[The computation times of the ground truth calculation and the proposed patched algorithm]{\includegraphics[width=.5\textwidth]{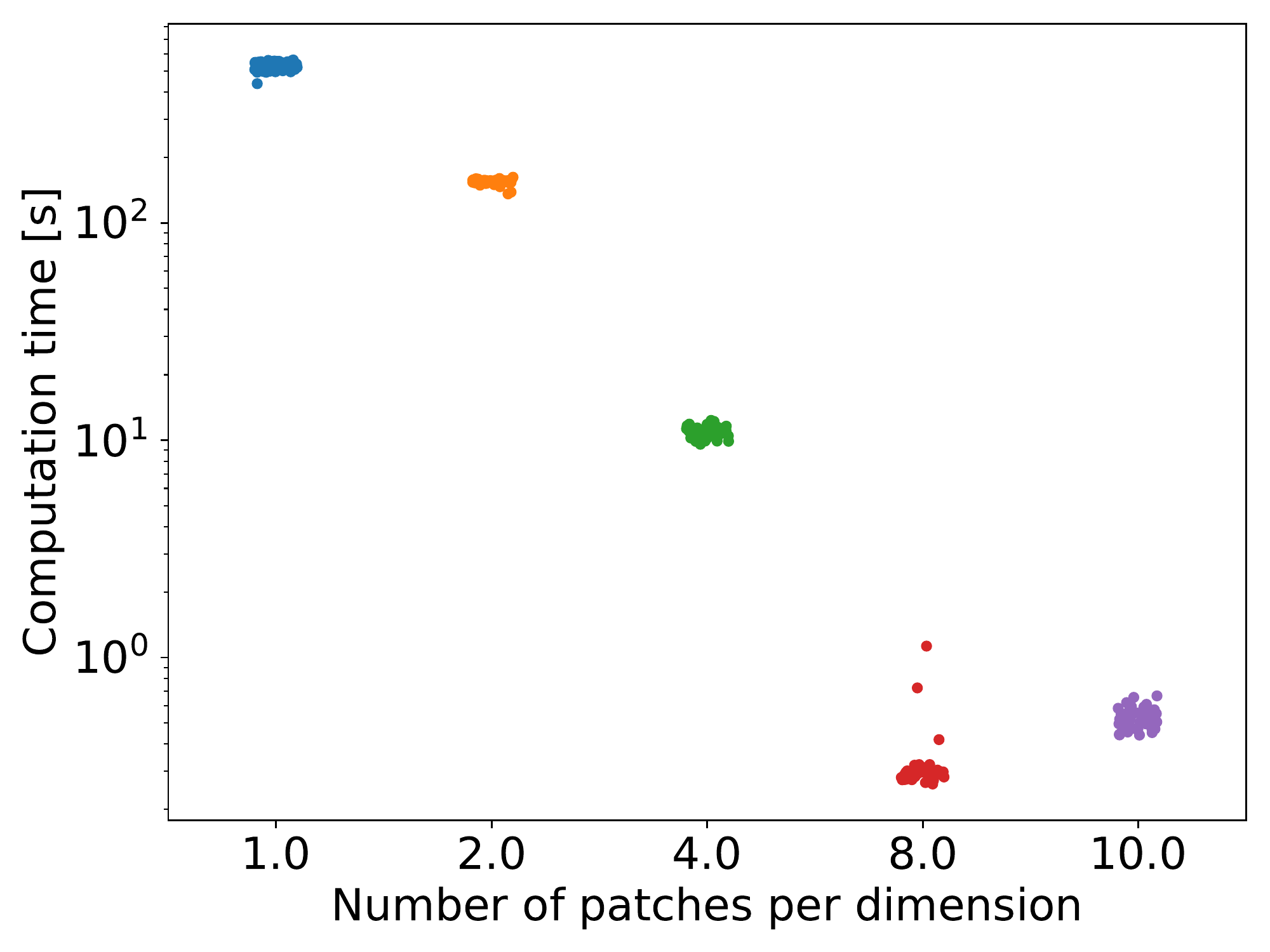}}
    \caption{Benchmark experiments performed on the $50$ test imaged of BIPEDv2. We test the computational speedup as well as the the approximation quality of Algorithm \ref{alg:speedup} and the rank-1 and local approximations outlined in Subsection \ref{subsec:2d_images}.}
    \label{fig:benchmark}
\end{figure}

As expected, the rank-1 approximation is by far the worst with a low average correlation between the ground truth and the approximated magnitude. A strong correlation is present in the local approximation, however, the maximum deviation and Fr\"obenius norm are still large compared to the patched approximations. The patched algorithm generally provides a good approximation to the ground trut with correlation values close to one and comparatively small absolute deviation and Fr\"obenius norm. The number of patches seems to have only a minor effect on the approximation accuracy, which is expected from the theoretical intuition behing Algorithm \ref{alg:speedup}. The computation time is drastically decreased by Algorithm \ref{alg:speedup}, with a smallest average computation time with patches of $25\times25$ pixels, resulting in $64$ patches in total, were considered. We attribute the increase in computation time for smaller patches to the computational overhead involved in processing a larger total number of patches.

\subsection{Learning the magnitude metric}\label{subsec:edge_detection}

Finally, we evaluate the capabilities of the magnitude vector as an edge detector. We compare our results against standard edge detection baselines such as the Sobel filter, the Canny edge detector, the ``vanilla magnitude'', the current state of the art deep learning models Dexined \citep{poma2020dense} and the context-aware tracing strategy (CATS)\citep{huan2021unmixing}.

\paragraph{Baselines:} The Sobel edge detector \citep{irwin1968isotropic} is a classical method used in computer vision. It relies on two convolution filters, the Sobel operators, to extract the gradients in the horizontal and vertical direction from the image. The edge maps, which correspond to an edge probability, are calculated by taking the absolute value of the gradient at each pixel. The first step to compute the Sobel edge map is to calculate a grey-scale image from the colour image using the formula
\begin{equation}
    c_\text{grey} = 0.2989c_\text{red} + 0.5870c_\text{green} + 0.1140c_\text{blue}.
\end{equation}
Then, we apply a Gaussian blur to the image and, finally, we use the Sobel operators with a given filter size.
As a postprocessing step we, again, use min-max scaling. The Gaussian and Sobel filter sizes are hyperparameters which can, in principle, be optimised. However, in our experiments we use a Gaussian filter size of $5$ and a Sobel filter size of $3$.

The Canny edge detector \citep{canny1986computational} builds on the Sobel filters and combines the gradient evaluation with a non-maximum suppression step, where the edge maps are sharpened, and a double-thresholding procedure to extract the edges. Unlike all the other methods considered here, the output of the Canny detector is a binary image, where pixel values of one correspond to edges and zeros correspond to non-edges. The thresholds and Sobel filter sizes are again hyperparameters which we optimise by requiring that the misclassification rate on the training set is minimised, i.e. the overlap of ones and zeros between the ground truth edge annotation and the Canny output is maximised.

The ``vanilla magnitude'' is simply the magnitude-transformed test image. Since an exact transformation is infeasible, we use Algorithm \ref{alg:speedup} to speed up the magnitude calculation. We use a Gaussian blur with filter size $5$ as a preprocessing and min-max scaling as a postprocessing step to obtain an edge probability map.

Both of the above algorithms are pixel-level methods which only use pixel intensities and their neighbours. Recently, deep learning methods for edge detection have also been developed. Most noteworthy are the Dexined \citep{poma2020dense} and the CATS \citep{huan2021unmixing} algorithms. Both rely on training a convolutional neural network in a supervised fashion, where the input is the image and the label is the edge map. What is special about these methods is that they can leverage information from the deeper layers in the neural network to obtain a better representation of the image. Since the magnitude vector calculation is also a pixel-level method we naturally expect the deep learning methods to outperform our approach.

\paragraph{Data preparation:} Since learning the magnitude vector on the full image is infeasible due to the prohibitive matrix inversion which needs to be performed, we leverage our proposed  Algorithm \ref{alg:speedup}. We first create patches from a given image and only learn the edge map on each patch of the image. Then, during testing, we transform each patch with the learned model and stitch the transformed patches together. We evaluate two scenarios:
\begin{enumerate}
    \item a random scenario, where we sample a random patch from each training image and
    \item a single-shot scenario, where we use all patches from one training image only as a training set.
\end{enumerate}
In each scenario, we set aside $20\%$ of the training patches as a validation set.

\begin{figure}[tb!]
\subfloat[The magnitude neural network module]{\includegraphics[width=\textwidth]{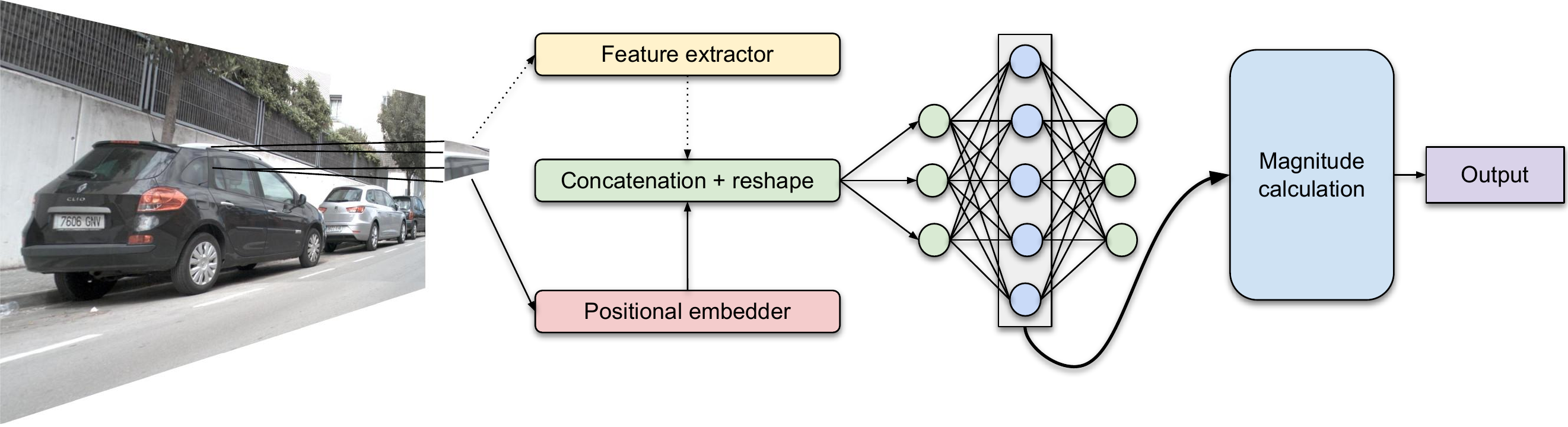}\label{subfig:edge_detector}}\\
\subfloat[The image transformer module]{\includegraphics[width=\textwidth]{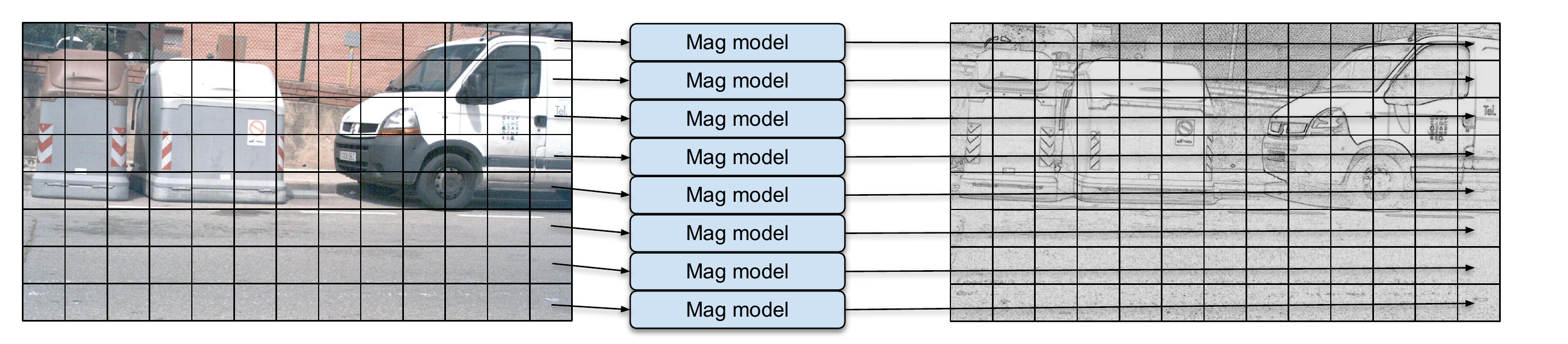}\label{subfig:image_transformer}}
\caption{A graphical overview of the magnitude edge detector. During training, we train the autoencoder (optionally the feature extractor) as presented in \ref{subfig:edge_detector} and during inference, we use the image transformer of \ref{subfig:image_transformer}}
\label{fig:edge_detection}
\end{figure}

\paragraph{Our model:} A graphical illustration of our edge detection model can be found in Figure \ref{fig:edge_detection}. Essentially, it consists of two parts, a trained model to find a good pixel embedding and an image transformer for inference. In the training step, we sample a patch from an image (with overlap to account for boundary effects) and reshape it such that we obtain a training batch of shape $(\#\text{pixels} \times \#\text{features})$. The features used are always the channel brightnesses ($= 3$ values: red, green, and blue) and the positional encoding as in Definition \ref{def:digital_image}. Additionally, we can create more features by using a ``feature extractor'' or backbone. In the forward step we feed this batch through the autoencdoer, as outlined in Subsection \ref{subsec:metric_learning}, and use the latent space representation to calculate the magnitude vector of the image patch. The image transformer transforms a full high-resolution image to a full-resolution magnitude vector using the generated latent-space representation of the image and Algorithm \ref{alg:speedup}. As postprocessing step we perform min-max scaling of the absolute values of the magnitude vector elements. The absolute value is taken, since numerical instabilities in the matrix inversion can lead to negative magnitude vector values.

\paragraph{Loss function:} As outlined in Subsection \ref{subsec:metric_learning} the loss function for our model is composed of the autoencoder loss and the magnitude loss discussed in \ref{subsec:metric_learning} as a regulariser for the autoencoder. In particular, for each patch we use the loss
\begin{equation}
    L = L_\text{AE} + \lambda L_\text{mag}.
\end{equation}
Here $$L_\text{AE} = \frac{1}{n}\sum_{i=1}^n (\bm{x}_1 - \hat{\bm{x}}_i)^2$$ is the autoencoder loss between a point $\bm{x}$ in the digital image and the reconstructed point $\hat{\bm{x}}$ and $n$ is the number of pixels in the patch. The term $L_\text{mag}$ is the loss between the magnitude-transformed patch and the ground truth edge map. Empirically we found, that the loss function $$L_\text{mag} = \frac{1}{|N_p|}\sum_{i\in N_p} (y_i-\hat{y}_i)^2 + \frac{1}{|N_n|}\sum_{i\in N_n} |y_i - \hat{y}_i|$$ is beneficial. The values $y \in [0,1]$ is a pixel in the ground-truth edge map and $\hat{y} \in \sR$ is a magnitude-transformed pixel and $N_p$ are the pixels for which $y=1$, whereas $N_n$ are pixels for which $y=0$. The distinction between edges ($y=1$) and non-edges ($y=0$) is necessary, since in each patch there are usually many more non-edges than edges and not discriminating between the two classes in the loss function results in inferior performance \citep{poma2020dense}.

During validation we use a modified magnitude loss, namely $$L_\text{mag} = \frac{1}{|N_p|}\sum_{i\in N_p} |y_i-\hat{y}_i| + \frac{1}{|N_n|}\sum_{i\in N_n} |y_i - \hat{y}_i|$$ to obtain a more fine-grained distinction bewteen the models.

\paragraph{Hyperparameters:} Naturally, there are many hyperparameters in our model that can be tuned such as the architecture of the autoencoder and the feature extractor, the patch size, the learning rate, the regularisation strength $\lambda$. In our experiments we use a learning rate of $0.001$, a patch size of $40\times 40$ and an overlap of $2$. The base metric is the $\ell_1$ metric as described in Subsection \ref{subsec:metric_learning}. We set the regularisation strength ($\lambda)$) to one, since the autoencoder loss always approaches $0$ and after a few epochs only the magnitude loss will be optimised. It is also not necessary to use weight decay or dropout. We do employ early stopping, however, we use model checkpointing which selects the model with the smallest validation loss for inference. We trained the model for $100$ epochs in the random scenario and for $50$ epochs in the single-shot scenario.

We study several different autoencoder architectures and also include a convolutional feature extractor. No hyperparamter tuning has been performed due to the difficulty of rigorous evaluation of model performance on the entire image. In particular, we propose three architectures:
\begin{itemize}
    \item Model I consists of a linear autoencoder with one layer transforming the initial $5$ features to $10$ features in the latent space. No feature extractor is used.
    \item Model II consits of a non-linear autoencoder with layer sizes $10,20$ and $40$ and a rectified linear unit (ReLU) activation function. Again, no feature extractor is used.
    \item Model III consists of a convolutional feature extractor with a single convolution layer with filter size $3$ transforming the $3$-channel input patch to a $15$-channel output. The resulting $20$ features are transformed using a linear autoencoder with one layer into a $40$ dimensional latent space.
\end{itemize}

\paragraph{Evaluation:} To rigorously evaluate our model performance, there are, in principle, two strategies; an indirect method, were the edge detection is evaluated via some downstream computer vision task, or a direct strategy, where the computed edge maps are compared directly to the ground truth annotation. We adopt the latter strategy in order to gain insight into the main advantages and disadvantages of the magnitude approach. To this end, we use four evaluation metrics commonly used in edge detection tasks \citep{poma2020dense}, namely:
\begin{enumerate}
    \item the Optimal Dataset Scale (ODS), where one threshold separating edges/non-edges is calculated for the entire dataset.
    \item the Optimal Image Scale (OIS), where one threshold is calculated per image.
    \item the Average Precision (AP).
    \item the average recall at $50\%$ precision (R50).
\end{enumerate}

In order to calculate these measures, we use standard postprocessing tasks such as non-maximal suppression (NMS) and morphological thinning of the edge maps.

\begin{figure}
    \centering
    \includegraphics[width=\textwidth]{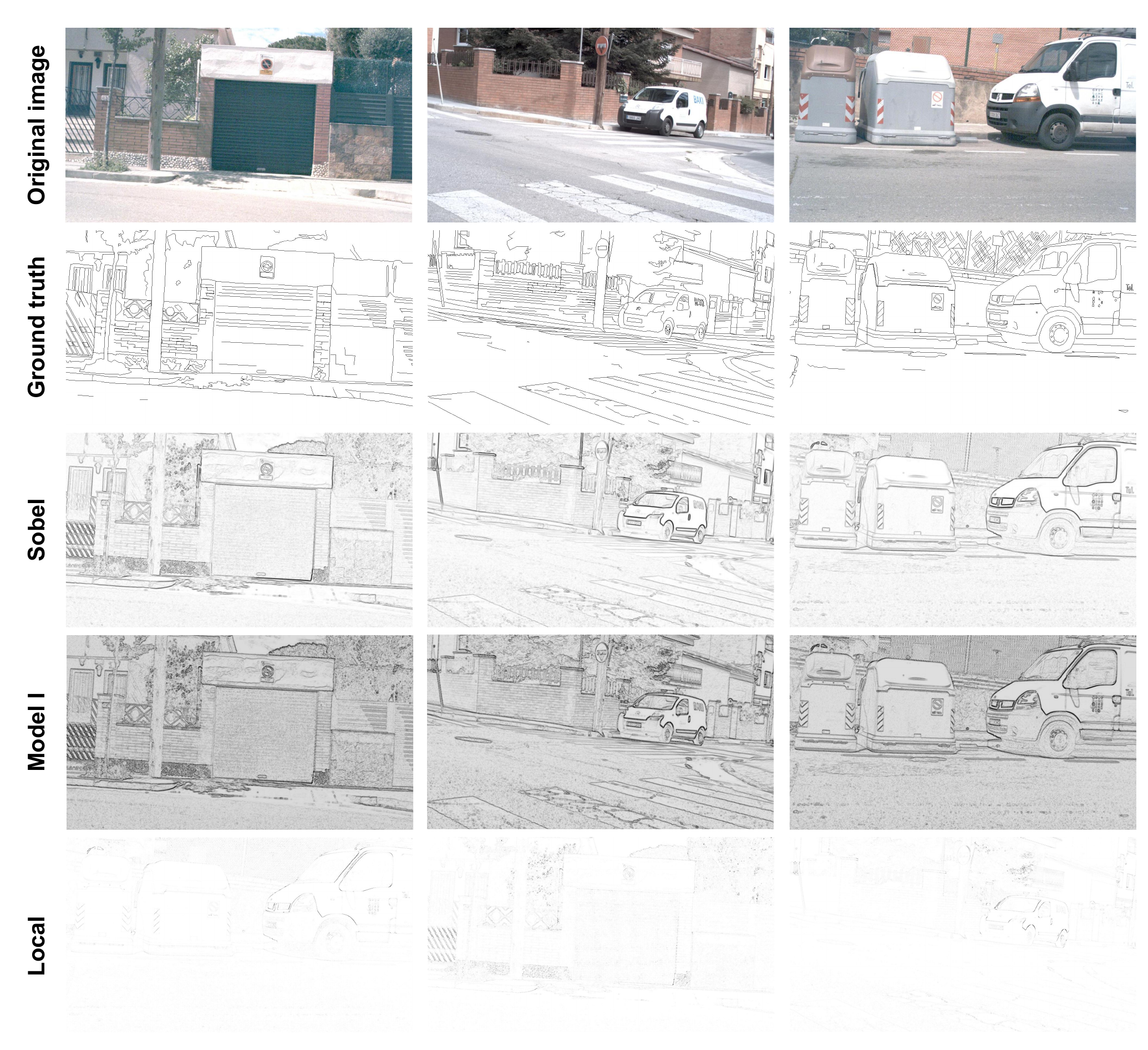}
    \caption{Example edge maps taken from the test set of BIPEDv2. We see that the ground truth annotation of the images not always exact and, therefore, any pixel-level evaluation should be taken with care. We compare the Sobel filter output, our best-performing magnitude model and the local approximation of the vanilla magnitude. The colours have been inverted for better visibility.}
    \label{fig:edge_maps}
\end{figure}

\paragraph{Topological properties of magnitude images:} Prior to performing a formalised investigation of the magnitude edge maps, we aim to quantify their structural---i.e.\ topological---properties. To this end, we leverage recent advances in computational topology and calculate the \emph{persistent homology} of the edge images that we obtain via the Sobel edge detector or our magnitude-based method. As shown by Hu et al.~\cite{Hu19a}, topological features are suitable to evaluate segmentations, for instance. Treating each image as a cubical complex~\cite{Rieck20}, we extract topological features in dimension~0~(connected components) and dimension~1~(cycles) of the resulting edge images; we summarise all features using the norm of their corresponding~\emph{Betti curve}~\cite{Rieck20c}, i.e.\ a simplified description of their topological complexity.
As Figure~\ref{fig:Magnitude_vs_Sobel} shows, the magnitude and Sobel edge images are structurally substantially different~(in terms of topological summary statistics).
Specifically, we see that magnitude exihibts a larger degree of computational complexity in terms of both connected components and cycles, thus underscoring the fact that Sobel edge images and magnitude edge images indeed \emph{capture qualitatively different structures}.

\begin{figure}[tbp]
    \centering
    \includegraphics{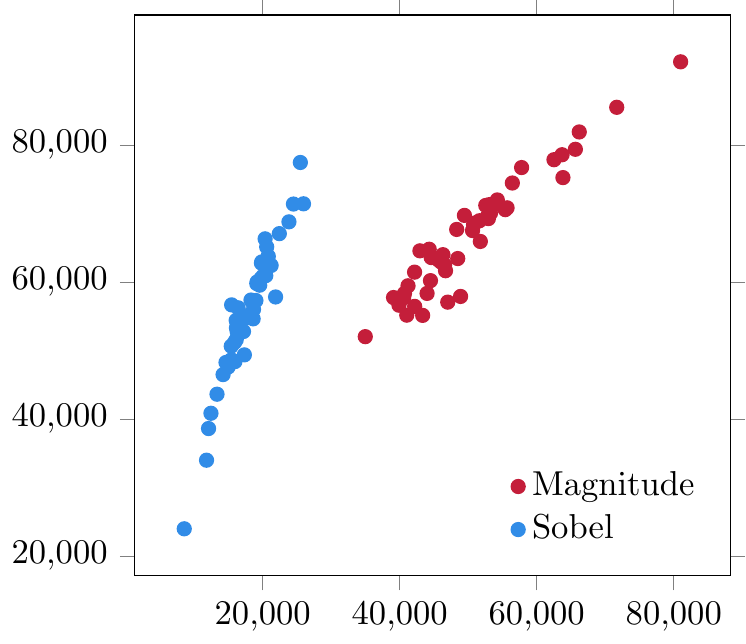}
    \caption{%
        Topological complexity in terms of the norm of Betti curves for magnitude~(red) and Sobel edge images~(blue).
    }
    \label{fig:Magnitude_vs_Sobel}
\end{figure}

\paragraph{Results:} The numerical results are presented in Table \ref{tab:results_ed} with example edge maps given in Figure \ref{fig:edge_maps}. As can be seen, the best magnitude model (Model I) is on par with the Sobel edge detector, potentially performing only slightly worse. This is surprising as, unlike the Sobel method, the magnitude is not a purpose-built edge detector. The Canny edge detector also did not perform well according to our evaluation metrics. However, this can be explained by the fact the Canny edge detector does \emph{not} provide a probabilistic edge map; rather it outputs a binary image with an already-implied hard decision boundary between edges and non-edges. Models II and III performed worst among all magnitude-based models. This could be attributed to the limited number of training examples used. The vanilla magnitude performed surprisingly well, although slightly worse than the optimised magnitude models. There seemed to be very little difference between the random and single-shot scenarios hinting at the fact that single images provides enough diversity to obtain a good latent space embedding. As expected, all pixel-level methods performed worse than the modern convolutional neural networks. This can be explained by the fact that more global information is available to these models. One noteworthy aspect is that the current model embeds single pixels into a latent space, therefore fine-tuning the distance between the pixels, but not taking into account more global information. This could be optimised in future versions of the magnitude edge detector. We also reiterate that \emph{no} rigorous optimisation of the autoencoder or feature-extractor architecture was performed. The results provide a clear indication that careful optimisation of these parameters can lead to substantial performance increases (or decreases).

Qualitatively, the Sobel and magnitude edge maps are very similar, however, the magnitude edge maps are usually slightly darker. This could also be attributed to our postprocessing steps, in particular, taking the absolute value of the magnitude vector elements before min-max scaling. Some details are lost in the local approximation, however, strong edges are still present.

\begin{table}
\centering
\begin{tabular}{lllll}
\hline
\textbf{Model}                            & \textbf{ODS}                              & \textbf{OIS}                              & \textbf{AP}                               & \textbf{R50}                               \\
\hline
Sobel                            & 0.766                            & 0.787                            & 0.842                            & 0.962                             \\
Canny                            & 0.458                            & 0.458                            & 0.0                              & 0.306                             \\
Vanilla Magnitude                & 0.706                            & 0.773                            & 0.780                            & 0.889                             \\
Model I                          & 0.739$\pm$0.013 & 0.761$\pm$0.011 & 0.805$\pm$0.012 & 0.938$\pm$0.010  \\
Model II                         & 0.587$\pm$0.106 & 0.665$\pm$0.084 & 0.534$\pm$0.256 & 0.621$\pm$0.300  \\
Model III                        & 0.733$\pm$0.016 & 0.758$\pm$0.014 & 0.775$\pm$0.021 & 0.946$\pm$0.010  \\
Model I Single Shot              & 0.737$\pm$0.013 & 0.758$\pm$0.012 & 0.804$\pm$0.013 & 0.938$\pm$0.010  \\
Model II Single Shot             & 0.683$\pm$0.038 & 0.707$\pm$0.036 & 0.752$\pm$0.039 & 0.859$\pm$0.062  \\
Model III Single Shot            & 0.730$\pm$0.020 & 0.754$\pm$0.017 & 0.771$\pm$0.033 & 0.941$\pm$0.011  \\
Dexined \citep{poma2020dense} & 0.895                            & 0.900                            & 0.927                            & N/A                               \\
CATS \citep{poma2020dense}    & 0.887                            & 0.892                            & 0.817                            & N/A\\
\hline
\end{tabular}
\caption{The edge detection performance of our magnitude model and baselines.}
\label{tab:results_ed}
\end{table}

\section{Conclusion}\label{sec:discussion}

In this paper, we introduced the magnitude vector of images. We started by outlining three different theoretical models for images and explained how they relate to each other. We then proceeded to show how the magnitude measure could be obtained for each of these image models and proved some foundational results.

The main theoretical contribution of this paper consists of explicitly deriving the magnitude measures for one-dimensional digitised images with $\ell_1$ metric. Specifically, we showed that the magnitude measure is mostly constant throughout the image, except at the step locations of step functions, where it is singular~(with a measure corresponding to the CDF of an exponential distribution). This allowed to theoretically motivate the ability of the magnitude measure to perform edge detection.
We also considered two-dimensional images. However, due to the analytical intractability, approximation strategies were introduced.
Based on these analytical results, we developed a patched speedup algorithm, which makes the magnitude vector calculation of high-resolution images feasible. We also considered refinements to the metric by introducing the notion of a pullback metric.

We performed a number of experiments validating our theoretical results, in particular, we empirically showed the validity of the patched magnitude calculation. In the final part of this paper, we presented results on the edge-detection capabilities of the magnitude vector with and without trained latent-space embeddings. The results of the experiments are promising and we found that the magnitude edge detector is approximately on par with the popular Sobel method, while still exhibiting substantially different topological, i.e.\ connectivity, properties.

These proof-of-principle experimental results open up a number of avenues for future research, both theoretically as well as experimentally. Major theoretical advances could consist of finding better approximations for the magnitude calculation in order to circumvent the matrix inversion step. Future experimental research could be directed towards finding better feature extractors or alternative metric learning procedures. In particular, it would be beneficial to be able to harness non-local pixel information.

\section*{Code and Data availability}

The code for our models can be found at \url{https://github.com/MikeAdamer/mag-metric}. The BIPED dataset is a public dataset and can be downloaded from \url{https://github.com/xavysp/MBIPED}.

\section*{Acknowledgements}

M.F.A.\ would like to thank Karsten Borgwardt for supporting this work.

\section*{Funding}

This project was funded in part by the Alfried Krupp
Prize for Young University Teachers of the Alfried Krupp von Bohlen und Halbach-Stiftung (M.F.A. via Karsten Borgwardt)

\section*{Conflict of interest}

There are no conflicts of interest.

\bibliography{main}
\bibliographystyle{unsrt}

\appendix

\end{document}